\documentclass[a4paper, 11pt]{article}

\usepackage{amsthm}
\usepackage{amsmath}
\usepackage{amssymb}
\usepackage{graphicx}
\usepackage[round,comma]{natbib}
\usepackage{bbm}
\usepackage{textcomp}
\usepackage{url}
\usepackage{comment}
\usepackage[margin=2.5cm]{geometry}

\theoremstyle{plain}  
\newtheorem{theorem}{Theorem}[section] 
\newtheorem{lemma}[theorem]{Lemma} 
\newtheorem{proposition}[theorem]{Proposition} 
\newtheorem{corollary}[theorem]{Corollary}

\theoremstyle{definition} 
\newtheorem{definition}{Definition}[section] 
 
\newtheorem{example}{Example}[section]

\theoremstyle{remark}

\newcommand{\diff}{\,\mathrm{d}}

\newcommand{\R}{\mathbb{R}}
\newcommand{\supp}{\operatorname{supp}}

\usepackage{color}

\newcommand{\atticY}{\mathcal{Y}}

\begin{document}

\title{Characteristic kernels on Hilbert spaces, Banach spaces, and on sets of measures}
\author{Johanna Ziegel, David Ginsbourger, and Lutz D\"umbgen\\[1ex]University of Bern\\[1ex]\url{{johanna.ziegel,david.ginsbourger,lutz.duembgen}@stat.unibe.ch}}

\maketitle

\begin{abstract}
We present new classes of positive definite kernels on non-standard spaces that are integrally strictly positive definite or characteristic. In particular, we discuss radial kernels on separable Hilbert spaces, and introduce broad classes of kernels on Banach spaces and on metric spaces of strong negative type. The general results are used to give explicit classes of kernels on separable $L^p$ spaces and on sets of measures.
\end{abstract}

\section{Introduction}

Symmetric, positive definite functions, so-called \emph{kernels}, are of great importance in many different areas. 
In spatial statistics, kernels are omnipresent as they appear as covariance functions of second order random fields \citep{Stein1999}. 
In machine learning, the ``kernel trick'' allows to generalize various methods from classical multivariate statistics pertaining to classification, dimensionality reduction, and regression by working in a reproducing kernel Hilbert space (RKHS) instead of the original input domain; see \citet{HofmannScholkopfETAL2008,SteinwartChristmann2008} and references therein. 
A related family of kernel methods relies on embeddings of probability measures into an RKHS, so-called kernel mean embeddings (KMEs): see \citet{BerlinetThomas-Agnan2004,SriperumbudurGrettonETAL2010}. Considering the RKHS-norm between the embedding of two measures yields a pseudometric on probability measures which is known as the maximum mean discrepancy (MMD). Its applications are manifold ranging from kernel two-sample tests \citep{GrettonBorgwardtETAL2012} over independence test \citep{GrettonBousquetETAL2005} to learning on distributional data \citep{Sutherland2016, SzaboSriperumbudurETAL2016}.
In statistical forecasting, kernel scores are important proper scoring rules that allow to compare the skill of competing probabilistic predictions \citep{Dawid2007,GneitingRaftery2007} with the most prominent example being the continuous ranked probability score \citep{MathesonWinkler1976}. The aim of this paper is to provide new classes of kernels on non-standard input spaces that are characteristic or even integrally strictly positive definite.

Let $(X,\mathcal{A})$ be a measurable space and $k$ a measurable kernel on $X$. The kernel score $S_k$ associated with $k$ is the map
\begin{equation}\label{eq:kernelscore}
S_k(P,x) = -\int_X k(\omega,x)\diff P(\omega) + \frac{1}{2}\int_X\int_X k(\omega,\omega') \diff P(\omega) \diff P(\omega') + \frac{1}{2}k(x,x), 
\end{equation}
for $P \in \mathcal{M}_1^k(X)$, $x \in X$, where $\mathcal{M}_1^k(X)$ is a suitable set of probability measures on $X$ such that all integrals are finite; see Section \ref{sec:prelim}. In contrast to \citet{Dawid2007}, the last term in \eqref{eq:kernelscore} is missing in the definition of kernel scores by \citet{GneitingRaftery2007}. We include it, since it ensures that the kernel score $S_k$ is non-negative.
By \citet[Theorem 1.1]{SteinwartZiegel2021}, $S_k$ is a proper scoring rule with respect to $\mathcal{M}_1^k(X)$, that is, for all $P,Q \in \mathcal{M}_1^k(X)$, $S_k(P,P) \le S_k(Q,P)$, where $S_k(Q,P) = \int_X S_k(Q,x)\diff P(x)$. The MMD $\gamma_k$ associated to $k$ is given by
\begin{equation}\label{eq:MMD}
\gamma_k(P,Q) = \sqrt{2(S_k(Q,P) - S_k(P,P))}, \quad P,Q \in \mathcal{M}_1^k(X),
\end{equation}
see \citet[Lemma 6]{GrettonBorgwardtETAL2012}. Each proper scoring rule $S$ gives rise to a divergence defined by $d(P,Q) = S(Q,P) - S(P,P)$ \citep{GneitingRaftery2007}. Equation \eqref{eq:MMD} shows that the divergences of kernel scores are squared MMDs, and vice versa.

Some MMDs have also been studied in statistics under the name of energy distances \citep{SzekelyRizzo2004,SzekelyRizzo2005,BaringhausFranz2004}, and they have been applied to various multivariate testing problems. Building on energy distances, \citet{SzekelyRizzoETAL2007,SzekelyRizzo2009} have proposed distance covariances, which have been studied further in probability theory \citep{Lyons2013} and statistics \citep{MatsuiMikoschETAL2017,BetkenDehling2021}. It should be noted that until the work of \citet{SejdinovicSriperumbudurETAL2013}, it was not clear that energy distances are in fact a special case of MMDs, see also the discussion before Theorem \ref{thm:Laplacekernel}.

Generally, MMDs are pseudometrics on probability measures, that is, they are symmetric and non-negative and they satisfy the triangle inequality; see Section \ref{sec:prelim} where it is explained how MMDs are even seminorms on some suitable vector space of signed measures. It may happen, however, that the MMD between two different probability measures is zero. Kernels for which the MMD is a metric (or a norm) are called \emph{characteristic}. A closely related notion is that of \emph{universality} of a kernel; see Section \ref{sec:prelim} for details. 
 
Characteristic and universal kernels are of particular interest in applications for various reasons: A characteristic kernel gives rise to a \emph{strictly} proper kernel score, that is, $S_k(P,P) = S_k(Q,P)$ implies $P=Q$, which encourages truthful reports of forecasters \citep{GneitingRaftery2007}. Equivalently, it yields MMDs that are metrics and not only pseudometrics, which is important for example for their application in two-sample tests. Universality of kernels is of fundamental importance for the statistical analysis of learning methods such as support vector machines. It allows to establish asymptotic consistency results \citep[Chapter 6 and 7]{SteinwartChristmann2008}.

Constructing positive definite kernels on arbitrary sets is a non-trivial task. However, if it can be accomplished, it opens the door to a plethora of methods including the ones mentioned above. By now, the literature is fairly rich concerning positive and strictly positive definite kernels on non-standard input spaces.
If these spaces are locally compact Hausdorff spaces, even characteristic and universal kernels are well-understood \citep{Steinwart2001,SriperumbudurGrettonETAL2010,Simon-GabrielScholkopf2018}. Of particular interest to this paper are the results of \citet{ChristmannSteinwart2010}. They provide universal kernels on general compact metric spaces. In particular, they provide kernels on the space of probability measures over a compact metric space. These kernels have found application in kernel-ridge regression with distribution-valued covariates \citep{SzaboSriperumbudurETAL2016}, see also \citet{MuandetFukumizuETAL2012}. 

If the input space is not locally compact, then the question of kernels being characteristic or universal is wide open. 
Just recently, \citet[Theorem 4]{HayatiFukumizuETAL2020} have shown the existence of a continuous characteristic kernel on an infinite dimensional separable Hilbert space. \citet{WynneDuncan2020} define a class called Squared Exponential $T$ (SE-T) kernels on a separable Hilbert space, and establish that for a nuclear, self-adjoint, and non-negative operator $T$, the resulting SE-T kernel is characteristic if only if $T$ is injective \citep[Theorem 3]{WynneDuncan2020}. Building on this result, they can also construct characteristic kernels on some Polish spaces.


The results that we present in this paper go further in several respects. On a separable Hilbert space $H$, we show in Section \ref{sec:Hilbert} that the Gaussian kernel is integrally strictly positive definite with respect to all finite signed measures on $H$, and extend this result to the whole class $\Phi_\infty^+$ of strictly positive definite radial kernels on $H$. In Section \ref{sec:Banach}, we introduce constructions of integrally strictly positive definite kernels on Banach spaces and on metric spaces of strong negative type. The first one builds on characteristic functions of Gaussian measures. The second one arises from chaining functions of $\Phi_\infty^+$ with the metric. Specific results for the case of separable $L^p$ spaces are presented in Section \ref{sec:Lp}, yielding two different classes of integrally strictly positive definite kernels using the first construction for $1<p<\infty$,  and the second construction for $1<p\leq 2$. Finally, in Section \ref{sec:measures}, we present results on integrally strictly positive definite kernels on sets of measures over locally compact Hausdorff spaces equipped with the topology of weak convergence. The considered classes of kernels rely on KMEs into RKHSs as well as into an $L^2$ space. In Section \ref{sec:prelim}, we introduce the necessary background and notation. While this section is quite technical, we hope that it makes the overall exposition close to self-contained. The paper concludes with a short discussion in Section \ref{sec:discussion}.

\section{Preliminaries}\label{sec:prelim}

\subsection{Kernel embeddings}
We follow the notation used in \citet{SteinwartZiegel2021}.
For a measurable space $(X,\mathcal{A})$, we denote by $\mathcal{M}(X)$ the space of all finite signed measures on $X$ and equip it with the total variation norm $\|\cdot\|_{TV}$. The set of probability measures is denoted by $\mathcal{M}_1(X)$, and the set of finite non-negative measures by $\mathcal{M}_+(X)$. We define $\mathcal{M}_0(X):=\{\mu \in \mathcal{M}\;|\; \mu(X) = 0\}$. For $\mu \in \mathcal{M}(X)$, we denote by $\mathcal{L}_1(\mu)$ the space of all measurable functions $f:X\to\mathbb{R}$ that are $\mu$-integrable. Let us recall the notion of Pettis integrals. Let $H$ be a Hilbert space and $f:X \to H$ a map. Then, $f$ is weakly measurable if $\langle w,f\rangle_H:X \to \mathbb{R}$ is measurable for all $w \in H$, and $f$ is weakly integrable with respect to $\mu$ if $\langle w,f\rangle_H \in \mathcal{L}_1(\mu)$ for all $w \in H$. If $f$ is weakly integrable, then there exists a unique $i_\mu(f) \in H$ such that for all $w \in H$, we have
\begin{equation}\label{eq:Pettis}
\langle w, i_\mu(f)\rangle_H = \int_X \langle w,f\rangle_H\diff \mu.
\end{equation}
The element $i_\mu(f)$ is called the Pettis integral of $f$ with respect to $\mu$ and we use the more intuitive notation $\int_X f \diff \mu := i_\mu(f)$ \citep[Chapter II.3]{DiestelUhl1977}.

Let $H$ be a Hilbert space of real-valued functions on $X$. A function $k:X\times X\to \mathbb{R}$ is a reproducing kernel of $H$ if $k(\cdot,x)\in H$ for all $x \in X$, and $\langle f, k(\cdot,x)\rangle_H = f(x)$ for all $f \in H$, $x \in X$. The Moore-Aronszajn theorem states that for every kernel $k$, that is, for every symmetric, positive definite function $k:X\times X \to \mathbb{R}$, there is an associated RKHS $H_k$ of real-valued functions on $X$ with reproducing kernel $k$ \citep[Chapter 4]{SteinwartChristmann2008}. The canonical feature map is $\Phi_k:X \to H_k, x \mapsto k(\cdot,x)$. The RKHS $H_k$ consists of measurable functions on $X$ if $k$ is measurable \citep[Lemma 4.24]{SteinwartChristmann2008}. Therefore, since $\langle f,\Phi_k\rangle_{H_k} = f$ for all $f \in H_k$, $\Phi_k$ is weakly measurable if $k$ is measurable. The canonical feature map $\Phi_k$ is weakly integrable with respect to $\mu \in \mathcal{M}(X)$ if and only if $f \in \mathcal{L}^1(\mu)$ for all $f \in H_k$. By \citet[Theorem 4.26]{SteinwartChristmann2008}, a sufficient condition for this is
\[
\int_X \sqrt{k(x,x)} \diff |\mu|(x) < \infty. 
\]
We define 
\[
\mathcal{M}^k(X) := \big\{\mu \in \mathcal{M}(X)\;|\; H_k \subseteq \mathcal{L}_1(\mu)\big\},
\]
and $\mathcal{M}_1^k(X) = \mathcal{M}^k(X) \cap \mathcal{M}_1(X)$. If $k$ is bounded, then $\mathcal{M}^k(X) = \mathcal{M}(X)$. It is also necessary for $\mathcal{M}^k(X) = \mathcal{M}(X)$ that $k$ is bounded, which can be seen by combining \citet[Proposition 2]{SriperumbudurGrettonETAL2010} and \citet[Lemma 4.23]{SteinwartChristmann2008} as remarked by \citet{SteinwartZiegel2021}. The vector space $\mathcal{M}^k(X)$ is the largest set on which the kernel mean embedding 
\[
\Phi_k(\mu) := \int_X \Phi_k \diff \mu = \int k(\cdot,x)\diff \mu(x)
\]
can be defined. The kernel mean embedding $\Phi_k\colon\mathcal{M}^k(X) \to H_k$ is linear, so 
\begin{equation}\label{eq:Hnorm}
\|\mu\|_{H_k} := \|\Phi_k(\mu)\|_{H_k}
\end{equation}
defines a (new) semi-norm on $\mathcal{M}^k(X)$. This semi-norm is a norm, if and only if the kernel mean embedding $\Phi_k$ is injective. If the kernel mean embedding $\Phi_k$ restricted to $\mathcal{M}_1^k(X)$ is injective, then the kernel $k$ is called \emph{characteristic}. Using \eqref{eq:Pettis}, we see that for $\mu_1,\mu_2 \in \mathcal{M}^k(X)$, it holds that
\[
\langle \Phi_k(\mu_1),\Phi_k(\mu_2)\rangle_{H_k} = \int_X\int_X k(x,x') \diff \mu_1(x) \diff \mu_2(x).
\]
Therefore, $\Phi_k$ is injective if and only if
\begin{equation}\label{eq:intspd}
\|\mu\|_{H_k}^2 = \int_X\int_X k(x,x')\diff \mu(x)\diff \mu(x') > 0
\end{equation}
for all $\mu \in \mathcal{M}^k(X)\backslash \{0\}$.
\begin{definition}
Let $k$ be a measurable kernel on $X$ and $\mathcal{M} \subseteq \mathcal{M}^k(X)$. Then, $k$ is called \emph{strictly integrally positive definite} with respect to $\mathcal{M}$ if $\|\mu\|_{H_k}^2 > 0$ for all $\mu \in \mathcal{M}\backslash \{0\}$. 
\end{definition}

\citet[Proposition 2.2]{SteinwartZiegel2021} show that a bounded measurable kernel $k$ on $X$ is characteristic if and only if it is strictly integrally positive definite with respect to $\mathcal{M}_0(X)$; see also \citet[Lemma 8]{SriperumbudurGrettonETAL2010}. The following lemma, generalizes \citet[Lemma 3.3]{SteinwartZiegel2021} to arbitrary convex mixtures of kernels which will be useful in subsequent arguments.

\begin{lemma}\label{lem:mixlemma}
Let $(X,\mathcal{A})$ be a measurable space, $\nu$ a finite measure on $\mathbb{R}$ with $\nu\not=0$, and $k_\theta$, $\theta \in \mathbb{R}$ a family of uniformly bounded kernels on $X$ with RKHSs $H_\theta$, $\theta \in \mathbb{R}$. We assume that $(x,y,\theta)\mapsto k_\theta(x,y)$ is $\mathcal{A}\otimes\mathcal{A}\otimes\mathcal{B}(\mathbb{R})$-measurable. Let $H$ be the RKHS of the kernel $k = \int_\mathbb{R} k_\theta \diff \nu(\theta)$. Then, for all $\mu \in \mathcal{M}(X)$, we have
\[
\|\mu\|_H^2 = \int \|\mu\|_{H_\theta}^2 \diff \nu(\theta). 
\]
If there is a Borel set $A \subseteq \mathbb{R}$ with $\nu(A) > 0$ such that $k_\theta$ is strictly integrally positive definite with respect to $\mathcal{M} \subseteq\mathcal{M}(X)$ for all $\theta \in A$, then $k$ is strictly integrally positive definite with respect to $\mathcal{M}$. In particular, if all $k_\theta$ have injective kernel mean embeddings, then so does $k$. If all $k_\theta$ are characteristic, then so is $k$.
\end{lemma}
\begin{proof}
By the definition of the norm $\|\cdot\|_H$ at \eqref{eq:Hnorm} and Fubini's theorem, we have that for all $\mu \in \mathcal{M}(X)$,
\begin{align*}
\|\mu\|_H^2 &= \int_X \int_X k(x,x') \diff \mu(x)\diff \mu(x') \\& = \int_X \int_X \int_\mathbb{R} k_\theta (x,x')\diff \nu(\theta) \diff \mu(x)\diff \mu(x')  \\&= \int_\mathbb{R} \int_X \int_X  k_\theta (x,x') \diff \mu(x)\diff \mu(x') \diff \nu(\theta) = \int \|\mu\|_{H_\theta}^2 \diff \nu(\theta).
\end{align*}
Let $\mu \in \mathcal{M}$ with $\mu \not=0$. If $\|\mu\|_{H_\theta} > 0$ for all $\theta$ in a set $A$ of strictly positive $\nu$-measure, then it follows that $\|\mu\|_H > 0$ by the formula we have just shown.
\end{proof}

\subsection{Further definitions and notation}

Let $(X,\tau)$ be a Hausdorff space and $\mathcal{B}(X)$ its Borel-$\sigma$-algebra. 
A Borel measure $\nu$ on $X$ is a non-negative measure on $\mathcal{B}(X)$ such that $\nu(K) < \infty$ for all compact sets $K \subseteq X$. A Borel measure $\nu$ is called a Radon measure if it is (inner and outer) regular. Let $\mathcal{M}^*(X)$ denote the space of all finite signed Radon measures on $X$ with the total variation norm. By Ulam's theorem \citep[Lemma 26.2]{Bauer2001}, if $X$ is a Polish space, then every finite Borel measure is a finite Radon measure. 

For a locally compact Hausdorff space $(X,\tau)$, let $C_0(X)$ be the space of continuous functions that vanish at infinity with the supremum norm. By Riesz's representation theorem \citep[Theorem 6.19]{Rudin1970}, the topological dual $C_0(X)'$ of $C_0(X)$ is equal to $\mathcal{M}^*(X)$. A continuous kernel $k$ on $X$ is called \emph{universal} if $H_k \subseteq C_0(X)$ and $H_k$ is dense in $C_0(X)$. A bounded continuous kernel $k$ on $X$ is called a $C_0(X)$-kernel if $H_k \subseteq C_0(X)$. \citet[Theorem 3.13]{SteinwartZiegel2021} show that a $C_0(X)$-kernel is universal if and only if it is strictly integrally positive definite with respect to $\mathcal{M}^*(X)$. If such a kernel exists, then $\mathcal{M}^*(X) = \mathcal{M}(X)$ by \citet[Theorem 3.12]{SteinwartZiegel2021}.

A Radon measure $\mathcal{W}$ on a Banach space $B$ is called a centered Radon Gaussian measure if the image measure $x'\#\mathcal{W}$ on $\mathbb{R}$ is centered Gaussian for all $x' \in B'$, where $B'$ denotes the topological dual of $B$; see \citet{Bogachev1998} for a comprehensive introduction to Gaussian measures on infinite dimensional spaces. For two Banach spaces $B$ and $C$, we denote the set of continuous linear operators from $B$ to $C$ by $L(B,C)$. For $x \in B$ and $x' \in B'$, we write $\langle x',x\rangle = x'(x)$ for the duality relation. We sometimes identify elements of $B$ with elements of $B''$ via the canonical embedding. 

For a normed vector space $(V,\|\cdot\|)$, a symmetric function $k$ on $V \times V$ is called radial if it is of the form $k(x,y) = \varphi(\|x-y\|^2)$, $x,y\in V$ for some function $\varphi:[0,\infty) \to \mathbb{R}$.

\section{Radial kernels on Hilbert spaces}\label{sec:Hilbert}

In this section, we characterize continuous integrally strictly positive definite radial kernels on separable Hilbert spaces 

Continuous (strictly) positive definite radial functions on Hilbert spaces have been characterized by \citet{Schoenberg1938} and \citet[Proposition 4]{BachocSuvorikovaETAL2020}. The latter reference extends Schoenberg's results for separable Hilbert spaces to general Hilbert spaces. For the continuous radial function $k$ to be positive definite, that is, a kernel, the function $\varphi$ has to be of the form
\begin{equation}\label{eq:varphi}
\varphi(t) = \int_{[0,\infty)} e^{-x t} \diff \nu(x), \quad t \ge 0,
\end{equation}
where $\nu$ is a finite Borel measure on $[0,\infty)$, that is, $\varphi$ is the Laplace transform of $\nu$. We denote the class of all Laplace transforms of finite Borel measures on $[0,\infty)$ by $\Phi_\infty$. If $\supp \nu \not=\{0\}$ and $\nu\not=0$, then $\varphi$ at \eqref{eq:varphi} induces a strictly positive definite function on $H$. We denote the class of all such functions with $\Phi_\infty^+$. The method of proof of \citet{BachocSuvorikovaETAL2020} does not use that the Hilbert space is complete. In fact, their argument works for any vector space $V$ with a scalar product $\langle \cdot,\cdot \rangle$, and functions of the form $\varphi(\langle x-y,x-y\rangle)$, $x,y \in X$. We make use of this fact in the proof of Theorem \ref{thm:Banachkernel}.

\citet[Theorem 2.1]{ChristmannSteinwart2010} show that the Gaussian kernel $\exp(-\|x-y\|_H^2/2)$, $x,y \in H$, is universal on every compact subset $K$ of a separable Hilbert space $H$, that is it is strictly integrally positive definite with respect to $\mathcal{M}(K)$. They suggest that their result can be generalized to hold for arbitrary $\varphi \in \Phi_\infty^+$ and this is indeed the case by Lemma \ref{lem:mixlemma}. In summary, every continuous strictly positive definite radial kernel is universal on every compact subset of $H$. 
We show that the restriction to compact subsets of $H$ is not necessary. 

\begin{theorem}\label{thm:Hilbertkernel}
	Let $H$ be a separable Hilbert space. Then the Gaussian kernel on $H$ is integrally strictly positive definite with respect to $\mathcal{M}(H)$. 
\end{theorem}

\citet[Proposition 6]{HayatiFukumizuETAL2020} have previously shown that the kernel mean embedding of the Gaussian kernel on $H$ is injective restricted to the family of Gaussian measures on $H$. \citet[Proposition 12]{WynneDuncan2020} show that the Gaussian kernel is characteristic on $H$. What distinguishes Theorem~\ref{thm:Hilbertkernel} from the latter result is the broader class of measures on which the associated kernel mean embedding is injective (integrally strictly positive definite versus characteristic) and the shorter proof provided for our more general result. Additionally, Lemma \ref{lem:mixlemma} yields the following immediate corollary.

\begin{corollary}\label{cor:Hilbertkernel}
Let $H$ be a separable Hilbert space and let $\varphi \in \Phi_\infty^+$. Then the kernel $k$ on $H$ defined by  
	\begin{equation*}\label{eq:Hilbertkernel}
	k(x,y) := \varphi\big(\|x-y\|^2\big), \quad x,y \in H,
	\end{equation*}
is integrally strictly positive definite with respect to $\mathcal{M}(H)$. 
\end{corollary}

The proof of Theorem \ref{thm:Hilbertkernel} relies on the following lemma, that generalizes a standard result in probability theory.  

\begin{lemma}\label{lem:mgfHilbert}
	Let $H$ be a separable Hilbert space and $\mu_1, \mu_2 \in \mathcal{M}_+(H)$. Define
\[
m_{j}(y)=\int \exp(\langle x, y \rangle) \diff\mu_{j}(x) \in [0,\infty]
\]	
for $j=1,2$ and arbitrary $x\in H$. Suppose that for some $\varepsilon > 0$, $m_1(y) = m_2(y) < \infty$ whenever $\|y\| \le \varepsilon$. Then $\mu_1=\mu_2$. 
\end{lemma}

\begin{proof}
Note that $m_j(0) = \mu_j(H)$, so $\mu_1(H) = \mu_2(H)$, and we may assume without loss of generality that $\mu_1, \mu_2 \in \mathcal{M}_1(H)$. If $\dim(H) = d < \infty$, we may identify $H$ with $\R^d$, equipped with its standard inner product. In that case, the result is well-known from probability theory.

Let $\dim(H) = \infty$. Separability of $H$ implies the existence of a complete orthonormal system $(\phi_i)_{i \in \mathbb{N}}$ in $H$. For $n \in \mathbb{N}$, let $\Pi_n$ be the orthogonal projection from $H$ onto the $n$-dimensional linear subspace $H_n := \operatorname{span}(\phi_1,\ldots,\phi_n)$. If $X_1$ and $X_2$ are random variables with distributions $\mu_1$ and $\mu_2$, respectively, then
\[
	m_j(y) = \mathbb{E} \exp(\langle X_j, y\rangle)
\]
for $j = 1,2$ and $y \in H$. Specifically, if $y \in H_n$ for some $n \in \mathbb{N}$, then
\[
	m_j(y) = \mathbb{E} \exp(\langle \Pi_n X_j, y\rangle) .
\]
Consequently, if $m_1(y) = m_2(y) < \infty$ whenever $\|y\| \le \varepsilon$, then $\Pi_n X_1$ and $\Pi_n X_2$ have the same distribution for any $n \in \mathbb{N}$. But this implies that $\mu_1 = \mu_2$, because $\Pi_n X_j \to X_j$ almost surely as $n \to \infty$ for $j = 1,2$, and hence, for any bounded continuous function $f : H \to \mathbb{R}$, by dominated convergence,
\[
	\int f(x) \diff\mu_1(x) 
	=  \lim_{n\to\infty} \mathbb{E} f(\Pi_n X_1) = \lim_{n\to\infty} \mathbb{E} f(\Pi_n X_2) 
	=  \int f(x) \diff\mu_2(x) .
\]
\end{proof}

\begin{proof}[Proof of Theorem~\ref{thm:Hilbertkernel}]
Let $\mu = \mu_+ - \mu_- \in \mathcal{M}(H)$ with $\mu_+,\mu_- \in \mathcal{M}_+(H)$. Then, if $k(x,y) = \exp(-\|x-y\|_H^2/2)$ is the Gaussian kernel and $y \in H$,
\[
	\Phi_k(\mu)(y) = \exp\Big(- \frac{\|y\|_H^2}{2}\Big)
		\Bigl( \int \exp(\langle x,y\rangle) \diff\tilde{\mu}_+(x)
			- \int \exp(\langle x,y\rangle) \diff\tilde{\mu}_-(x) \Bigr)
\]
with the finite measures $\tilde{\mu}_+, \tilde{\mu}_-$ given by $\diff\tilde{\mu}_j(x) := \exp(-\|x\|_2^2/2) \diff\mu(x)$. Thus, $\Phi_k(\mu) = 0$ is equivalent to
\[
	\int \exp(\langle x,y\rangle_{H}^{}) \diff\tilde{\mu}_+(x) 
	=  \int \exp(\langle x,y\rangle_{H}^{}) \diff\tilde{\mu}_-(x), \quad \text{for all $y \in H$.}
\]
By definition of the measures $\tilde{\mu}_\pm$, the latter integrals are finite. It follows from Lemma~\ref{lem:mgfHilbert} that $\tilde{\mu}_+ = \tilde{\mu}_-$. But $\mu_\pm$ may be represented as $\mu_\pm(dx) = \exp(\|x\|_{H}^2/2) \diff \tilde{\mu}_\pm(x)$, so $\mu_+ = \mu_-$, that is, $\mu = 0$.
\end{proof}

While we focus on radial kernels in this section, our results carry over to kernels obtained by composing arguments via suitable classes of maps, in the spirit of the SE-T kernels of \citet{WynneDuncan2020}. In particular, Proposition~\ref{prop:genT4_WandD} shows that \citet[Theorem 4]{WynneDuncan2020} holds when replacing the squared exponential kernel with any radial kernel in $\Phi_\infty^+$, thus straightforwardly proving \citet[Corollary 1]{WynneDuncan2020}.

\section{Characteristic kernels on Banach spaces}

\subsection{General results}\label{sec:Banach}

\citet[Section 5.2]{BergChristensenETAL1984} discuss the question whether the Gaussian kernel, that is, $t \mapsto \exp(-t^2)$ leads to a radial kernel on a Banach space $B$. The answer is negative in the following sense \citep[Lemma 5.2.2.2]{BergChristensenETAL1984}: If $k(x,y) = \exp(-\lambda \|x-y\|^2)$ is positive definite on $B$, then $B$ is already a Hilbert space. Therefore, there is no hope for a generalization of Theorem \ref{thm:Hilbertkernel} to Banach spaces. 

On the Banach space $C([0,1])$ of continuous functions on $[0,1]$ with the supremum norm $\|\cdot\|_{\infty}$, the situation is even more drastic in the sense that $\varphi(\|f-g\|_\infty)$, $f,g \in C([0,1])$ is not positive definite for any $\varphi \in \Phi_\infty^+$ \citep[Theorem 5.3.2.3]{BergChristensenETAL1984}. An intermediate situation is encountered for $L^p$-spaces, where the Gaussian kernel does not lead to a radial kernel for $p\not=2$ but the Laplace kernel, that is, $\exp(-\|f-g\|_{L^p})$, $f,g \in L^p$ does; compare Corollary \ref{cor:LpLaplace}.

In this section, we provide two construction principles for integrally strictly positive definite or characteristic kernels on Banach spaces. The first one (Theorem \ref{thm:Banachkernel}) yields kernels that are similar to Gaussian kernels but with respect to a different norm, whereas the second one yields variants of Laplace kernels (Theorem \ref{thm:Laplacekernel}). 

\begin{theorem}\label{thm:Banachkernel}
Let $B$ be a Banach space and $A \in L(B'',B')$ an operator that induces a non-degenerate centered Radon Gaussian measure on $B'$ with covariance $A$, and let $\varphi \in \Phi_\infty^+$. Then, 
\begin{equation}\label{eq:Banachkernel}
k(x,y) := \varphi\Big(\langle A(x-y),x-y\rangle\Big), \quad x,y \in B \subseteq B'',
\end{equation}
is a kernel on $B$ that is integrally strictly positive definite with respect to $\mathcal{M}^*(B)$.
\end{theorem}
\begin{proof} 
Showing that $k$ is strictly positive definite on $B$ works with the same arguments as \citet[Propositions 2-4]{BachocSuvorikovaETAL2020}.
By definition, the centered Gaussian measure $\mathcal{W}$ on $B'$ with covariance operator $A$ has characteristic function
\begin{equation}\label{eq:charfunc}
	\int_{B'}\exp(i \langle x'',x'\rangle)\diff \mathcal{W}(x')
	= \exp\Big(-\frac{\langle Ax'',x''\rangle}{2}\Big), \quad x'' \in B''.
\end{equation}
Since $\mathcal{W}$ is assumed to be non-degenerate, 
it has full support, 
that is, it assigns positive probability to every non-empty open subset of $B'$. 
Let $\mu \in \mathcal{M}^*(B)$. Then, using \eqref{eq:charfunc} and Fubini's theorem,
\begin{align*}
\|\mu\|_{H_k}^2 &= \int_{B} \int_{B} \exp\Big(-\frac{\langle A(x-y),x-y\rangle}{2}\Big) \diff \mu(x)\diff \mu(y)\\
&= \int_{B'} \int_{B} \int_{B}\exp(i \langle x',x-y\rangle) \diff \mu(x)\diff \mu(y) \diff\mathcal{W}(x')\\
&= \int_{B'} \bigg|\int_{B} \exp(i \langle x',x\rangle) \diff \mu(x)\bigg|^2 \diff\mathcal{W}(x').
\end{align*}
Therefore, if $\|\mu\|_{H_k}^2=0$, then $\hat{\mu}(x'):=\int_{B} \exp(i \langle x',x\rangle) \diff \mu(x) = 0$ for $\mathcal{W}$-almost all $x' \in B'$. The map $\hat{\mu}$ is continuous with respect to weak convergence on $B'$, and therefore also with respect to the operator norm convergence on $B'$. This implies that $\hat{\mu}(x') = 0$ for all $x' \in B'$. Since $\hat{\mu}(0) = \mu^+(B) - \mu^-(B) = 0$, we see that $\mu \in \mathcal{M}_0(B)$, and it is no loss of generality to assume that $\mu^+,\mu^- \in \mathcal{M}_1(B)$. We obtain that the characteristic functions (or Fourier transforms) of $\mu^+$ and $\mu^-$ are equal, and therefore, $\mu = \mu^+ - \mu^- = 0$ \citep[Proposition A.3.18]{Bogachev1998}. Using the covariance operator $\alpha A$ for any $\alpha > 0$ instead of $A$, it follows that the kernel at \eqref{eq:Banachkernel} with $\varphi(t) = \exp(-\alpha t^2)$ is integrally strictly positive definite with respect to $\mathcal{M}^*(B)$. The claim now follows from Lemma \ref{lem:mixlemma}.
\end{proof}

The kernels defined at \eqref{eq:Banachkernel} in Theorem \ref{thm:Banachkernel} are not radial with respect to the norm of the Banach space $B$. However, they are radial with respect to a different norm on $B$ given by $\|x-y\|_A = \sqrt{\langle A(x-y),x-y\rangle}$, $x,y \in B$, with respect to which $B$ is generally not complete.

Recently, \citet{PanTianETAL2018} have suggested the ball divergence, which is a divergence between probability measures on separable Banach spaces. If one of the measures has full support, then the divergence can only be zero if the two measures are identical \citep[Theorem 2]{PanTianETAL2018}. Their divergence cannot be the squared MMDs induced by a kernel on the Banach space since its square root does not satisfy the triangle inequality. However, it is closely related to an MMD \citep[Remarks 2.2 and 2.4]{PanTianETAL2018}.

Characterizing operators $A \in L(B'',B')$ that can be obtained as covariances of Gaussian measures on $B'$ is generally difficult. But there are some notable exceptions such as separable $L^p$-spaces which we treat in Section \ref{sec:Lp}. If $B = H$ is a separable Hilbert space, then the operators that induce non-degenerate centered Gaussian measures are the nuclear, self-adjoint strictly positive operators \citep[Corollary 2.3]{Baxendale1976}. However, Corollary \ref{cor:Hilbertkernel} is stronger than Theorem \ref{thm:Banachkernel} in this case. This can be seen as follows. Let $T:H \to H$ be injective and measurable, $\varphi \in \Phi_\infty^+$ and let $\mu \in \mathcal{M}(H)$. Then, denoting the image measure of $\mu$ under $T$ by $T\#\mu$,
\[
0 = \int \varphi\big(\|T(x-y)\|^2\big)\diff \mu(x)\diff \mu(y) = \int \varphi\big(\|x-y\|^2\big)\diff T\#\mu(x)\diff T\#\mu(y),
\]
implies that $T\#\mu = 0$ by Corollary \ref{cor:Hilbertkernel}. Since $T$ is injective and measurable, it is bi-measurable by Purves' theorem \citep{Purves1966}, that is, we have $T(M) \in \mathcal{B}(H)$ and $T^{-1}(T(M)) = M$. This implies that $\mu = 0$. Taking $T$ to be a Hilbert-Schmidt operator, we obtain the result of Theorem \ref{thm:Banachkernel} for separable Hilbert spaces. 

On some Banach space radial characteristic kernels exist. We show a more general result. Let $(X,\rho)$ be a metric space. It is said to be of \emph{negative type} if $\rho$ is negative definite in the sense of \citep[Definition 3.1.1]{BergChristensenETAL1984}, that is, if for all $n \in \mathbb{N}$, $c_1,\dots,c_n \in \mathbb{R}$ with $\sum_{i=1}^n c_i = 0$ and $x_1,\dots,x_n \in X$, it holds that $\sum_{i,j=1}^n c_i c_j \rho(x_i,x_j) \le 0$. A metric $\rho$ of negative type induces a positive definite kernel $k$ on $X$ by defining $k(x,y) = \rho(x,z_0) + \rho(y,z_0) - \rho(x,y)$ for some fixed $z_0 \in X$ \citep[Lemma 3.2.1]{BergChristensenETAL1984}. The relation between the metric $\rho$ and the induced so-called distance kernel $k$ has been studied in detail by \citet{SejdinovicSriperumbudurETAL2013}. In particular, they show that (under suitable integrability assumptions) the energy distance with respect to $\rho$ is the same as the squared MMD with respect to $k$ \citep[Theorem 22]{SejdinovicSriperumbudurETAL2013}. 
Define
\[
\mathcal{M}_\rho(X) = \{\mu \in \mathcal{M}(X) \;|\; \exists z_0 \in X \;\text{such that}\; \int \rho(x,z_0)\diff |\mu|(x) < \infty\}.
\]
Following \citet{Lyons2013}, the metric space is of \emph{strong negative type} if for all $\mu \in \mathcal{M}_0(X) \cap \mathcal{M}_\rho(X)$, it holds that
\[
\int\int \rho(x,y)\diff \mu(x) \diff \mu(y) = 0 \quad \text{implies that} \quad \mu = 0. 
\]
The notion of a metric of strong negative type \citep[Section 3]{Lyons2013} is the same as requiring the associated distance kernel $k$ to be characteristic (restricted to probability measures $\mu$ that satisfy $\int_X k(x,x) \diff \mu(x) < \infty$).

\begin{theorem}\label{thm:Laplacekernel}
Let $(X,\rho)$ be a metric space of strong negative type and $\varphi \in \Phi_\infty^+$. Then, 
\[
k(x,y):= \varphi(\rho(x,y)), \quad x,y \in X,
\]
is a kernel on $X$ that is characteristic, that is, strictly integrally positive definite with respect to $\mathcal{M}_0(X)$.
\end{theorem}
\begin{proof}
Due to Lemma \ref{lem:mixlemma}, it suffices to show that the Laplace kernel $\exp(-\rho(x,y))$, $x,y \in X$, is strictly integrally positive definite with respect to $\mathcal{M}_0(X)$. Let $z_0 \in X$ and define $k(x,y) = \rho(x,z_0) + \rho(y,z_0) - \rho(x,y)$, $x,y \in X$. Since $\rho$ is a metric, $k$ is non-negative. By \citet[Lemma 3.2.1]{BergChristensenETAL1984}, $k$ is a kernel on $X$, and for any $\ell \in \mathbb{N}$, $k^\ell$ is also a kernel.
For $\mu \in \mathcal{M}_0(X)$, we obtain that
\begin{align*}
	\int\int \exp(-\rho(x,y)) \diff\mu(x)\diff\mu(y)
	&= \int \int \exp(k(x,y)) \diff\tilde\mu(x)\diff\tilde\mu(y) \\
	&= \sum_{\ell = 0}^\infty \frac{1}{\ell! } \int \int k(x,y)^\ell
		\diff \tilde\mu(x)\diff\tilde\mu(y),
\end{align*}
using monotone convergence, where the measure $\tilde\mu \in \mathcal{M}_0(X)$ is given by $\tilde{\mu} = \tilde\mu_+ - \tilde\mu_-$ with $\diff\tilde\mu_\pm(x) = \exp(-\rho(x,z_0))\diff \mu_\pm(x)$ and $\mu_\pm \in \mathcal{M}_+(X)$ are such that $\mu = \mu_+ - \mu_-$. Note also that $\int \sqrt{\exp(k(x,x)} \diff|\tilde{\mu}|(x) = |\tilde{\mu}|(X) < \infty$.
Since all summands on the right hand side are non-negative, $\int\int \exp(-\rho(x,y))\diff \mu(x)\diff\mu(y) = 0$ implies that $\tilde{\mu}(X) = 0$, and
\[
	0 = \int \int k(x,y) \diff\tilde\mu(x)\diff\tilde\mu(y)
	= - \int \rho(x,y) \diff\tilde{\mu}(x)\diff\tilde{\mu}(y) .
\]
Since $(X,\rho)$ is of strong negative type, we find that $\tilde \mu = 0$ which allows to conclude that $\mu = 0$.
\end{proof}

\subsection{Kernels on $L^p$-spaces}\label{sec:Lp}

We consider now the special case of separable $L^p$-spaces, $1 < p < \infty$. In this case, we can derive explicit expressions for integrally strictly positive definite kernels using Theorems \ref{thm:Banachkernel} and \ref{thm:Laplacekernel}. For the first construction, the crucial ingredient is that characterizations of covariance operators on the dual are available in this case.

The kernels presented in this section may find application in the construction of new kernel two-sample tests for functional data, see \citet{WynneDuncan2020} for exisiting work in this direction. They can also be used to construct kernel scores for probabilistic predictions of outcomes taking functional values such as precipitation fields over spatial domains.

Let $(X,\mathcal{A})$ be a measurable space and and $\lambda$ a finite or $\sigma$-finite measure on $(X,\mathcal{A})$ such that $L^p(\lambda)$ is separable. Let $1 < q < \infty$ such that $1/p + 1/q = 1$. Then, also $L^q(\lambda)$ is separable. If $A \in L(L^p(\lambda),L^q(\lambda))$ is the covariance of a centered Gaussian measure on $L^q(\lambda)$, then by \citet[Proposition 3.11.15]{Bogachev1998}, there exists a measurable kernel $k_1$ on $X$ with 
\begin{equation}\label{eq:kernelqint}
\int_X k_1(x,x)^{\frac{q}{2}}\diff\lambda(x) < \infty
\end{equation}
and
\begin{equation}\label{eq:covopk1}
(Af)(x) = \int_X k_1(x,y) f(y) \diff \lambda(y), \quad f \in L^p(\lambda),\; x \in X.
\end{equation}
Conversely, any measurable kernel $k_1$ on $X$ that satisfies \eqref{eq:kernelqint} induces a covariance of some centered Gaussian measure via \eqref{eq:covopk1}.
This leads to the following corollary to Theorem \ref{thm:Banachkernel}.
\begin{corollary}\label{cor:Lpkernel}
Let $1 < p < \infty$. Suppose that $L^p(\lambda)$ is separable and that $k_1$ is a measurable kernel on $X$ that satisfies \eqref{eq:kernelqint}, and 
\begin{equation}\label{eq:Lpintpd}
\int_X \int_X k_1(x,y) g(x)g(y)\diff \lambda(y)\diff\lambda(x) > 0 \quad \text{for all $g \in L^p(\lambda)\backslash \{0\}$.}
\end{equation}
Let $\varphi \in \Phi_\infty^+$. Then, 
\begin{equation*}\label{eq:Lpkernel}
k_2(f,g) := \varphi\Big(\int_X\int_X k_1(x,y)(f(x)-g(x))(f(y)-g(y))\diff \lambda(y)\diff\lambda(x) \Big), \quad f,g \in L^p(\lambda),
\end{equation*}
is a kernel on $L^p(\lambda)$ that is integrally strictly positive definite with respect to $\mathcal{M}(L^p(\lambda))$.
\end{corollary}
\begin{proof}
We need to show that $k_1$ induces a \emph{nondegenerate} centered Gaussian measure on $L^q(\lambda)$. A centered Gaussian measure $\tau$ on $L^q(\lambda)$ is nondegenerate if and only if its Cameron-Martin $H(\tau)$ space is dense in $L^q(\lambda)$. The space $H(\tau)$ contains the image of the operator $A$ defined at \eqref{eq:covopk1} \citep[Theorem 3.2.3]{Bogachev1998}. By \citet[Satz III.4.5]{Werner2002}, the image of $A$ is dense in $L^q(\lambda)$ if and only if the null space of the adjoint operator $A'=A\in L(L^p(\lambda),L^q(\lambda))$ is trivial. This in turn is equivalent to the following condition for any $g \in L^p(\lambda)$:
\[
\int_X \int_X k_1(x,y) g(x) f(y)\diff \lambda(y)\diff\lambda(x) =0 \quad \text{for all $f \in L^p(\lambda)$} \quad \implies \quad g = 0.
\]
This condition is implied by \eqref{eq:Lpintpd}.
\end{proof}

The condition at \eqref{eq:Lpintpd} is generally different to the kernel $k_1$ being integrally strictly positive definite with respect to $\mathcal{M}(X)$. This is easily seen by choosing $\lambda$ to be a measure with finite support. \citet{SriperumbudurGrettonETAL2010} consider the definition of an integrally strictly positive definite kernel a generalization of the condition $\int_{\R^d}\int_{\R^d}k_1(x,y)g(x)g(y)\diff x \diff y > 0$, $g \in L^2(\R^d)$ given by \citet[Section 6]{Stewart1976}. The cases $p \in \{1,\infty\}$ are not included in Corollary \ref{cor:Lpkernel} since $L^p$ is not reflexive in this case.


Interestingly, Corollary~\ref{cor:Lpkernel} also comes as an echo to \citet[Theorem 4]{WynneDuncan2020} and its generalization given by Proposition~\ref{prop:genT4_WandD}, by taking for $\atticY$ the RKHS $H_1$ associated with $k_1$ and for $T$ the operator $A$ defined by \eqref{eq:covopk1} seen as a mappping from $L^p(\lambda)$ to $H_1$. 
The condition at \eqref{eq:Lpintpd} then ensures that the latter operator is injective.   

The norm of separable $L^p$ spaces for $1 \le p \le 2$ is of negative type and even of strong negative type for $1 < p \le 2$ \citep{Linde1986b,Lyons2013}. Therefore, we obtain the following corollary to Theorem \ref{thm:Laplacekernel}. 

\begin{corollary}\label{cor:LpLaplace}
Suppose that $L^p(\lambda)$ is separable and $1 < p \le 2$. Let $\varphi \in \Phi_\infty^+$. Then, 
\[
k(x,y):= \varphi(\|f-g\|_{L^p}), \quad f,g \in L^{p}(\lambda),
\]
is a kernel on $L^{p}(\lambda)$ that is characteristic, that is, strictly integrally positive definite with respect to $\mathcal{M}_0(L^{p}(\lambda))$.
\end{corollary}


\section{Kernels on measures}\label{sec:measures}

In this section, we use Corollary \ref{cor:Hilbertkernel} to construct integrally strictly positive definite kernels on spaces of measures over locally compact Hausdorff spaces. Given a locally compact Hausdorff space $(X,\tau)$ and a characteristic kernel $k_1$ on $X$, we can injectively embed $\mathcal{M}_1(X)$ in the Hilbert space $H_{k_1}$ via the kernel mean embedding $\Phi_{k_1}$. Now, we can apply Corollary \ref{cor:Hilbertkernel} to the image of $\mathcal{M}_1(X)$ under $\Phi_{k_1}$. The essential idea of this construction has appeared previously in the literature, see for example \citet[Example 1]{ChristmannSteinwart2010}; \citet{BuathongGinsbourgerETAL2020}; \citet{HamidSchulzeETAL2021}. To the best of our knowledge, \citet{ChristmannSteinwart2010} are the only ones to consider universality of the resulting kernels, and not only (strict) positive definiteness. In contrast to their result, the following theorem does not require compactness of $X$. Below, we discuss the relation of Theorem \ref{thm:deepkernelM1} in the special case of compact $X$ to the result of \citet{ChristmannSteinwart2010}.

\begin{theorem}\label{thm:deepkernelM1}
Let $(X,\tau)$ be a locally compact Hausdorff space, and let $k_1$ a bounded continuous kernel on $X$ that is integrally strictly positive definite with respect to $\mathcal{M}(X)$ and with separable RKHS $H_{k_1}$. Then, for any $\varphi \in \Phi_\infty^+$, the kernel
\begin{equation*}\label{eq:k2}
k_2(\mu,\nu):= \varphi\Big(\int_X\int_X k_1(x,x')\diff(\mu-\nu)(x)\diff(\mu-\nu)(x')\Big), \quad \mu,\nu \in \mathcal{M}_+(X),
\end{equation*}
is integrally strictly positive definite with respect to $\mathcal{M}(\mathcal{M}_1(X))$, where $\mathcal{M}_1(X)$ is equipped with the topology of weak convergence of probability measures. 

If $X$ is additionally Polish, then $k_2$ is integrally strictly positive definite with respect to $\mathcal{M}(\mathcal{M}_+(X))$, where $\mathcal{M}_+(X)$ is equipped with the topology of weak convergence of measures  \citep[Definition 30.7]{Bauer2001}.
\end{theorem}
\begin{proof}
The kernel $k_2$ is the composition of $\tilde{k}_2$ with the kernel mean embedding $\Phi_{k_1}\colon\mathcal{M}_+(X)\to H_{k_1}$, where 
\[
\tilde{k}_2(f,f') := \varphi\big(\|f-f'\|^2_{H_{k_1}}\big), \quad f,f' \in H_{k_1}.
\]
Since $H_{k_1}$ is assumed to be separable, Corollary \ref{cor:Hilbertkernel} shows that $\tilde{k}_2$ is a kernel on $H_{k_1}$ that is integrally strictly positive definite with respect to $\mathcal{M}(H_{k_1})$. This immediately implies that $k_2$ is a kernel on $\mathcal{M}_+(X)$. 

Let $\Xi \in \mathcal{M}(\mathcal{M}_+(X))$. By \citet[Lemma 10]{Simon-GabrielScholkopf2018}, the kernel mean embedding $\Phi_{k_1}:\mathcal{M}_+(X) \to H_{k_1}$ is continuous with respect to weak convergence and thus measurable. Therefore, we can consider the image measure $\Phi_{k_1}\#\Xi \in \mathcal{M}(H_{k_1})$. We obtain that
\begin{multline*}
\int_{\mathcal{M}_+(X)}\int_{\mathcal{M}_+(X)} \varphi\big(\|\Phi_{k_1}(\mu)-\Phi_{k_1}(\mu')\|^2_{H_{k_1}})\diff \Xi(\mu)\diff \Xi(\mu')  \\
= \int_{H_{k_1}}\int_{H_{k_1}} \varphi\big(\|f-f'\|^2_{H_{k_1}})\diff (\Phi_{k_1}\#\Xi)(f)\diff (\Phi_{k_1}\#\Xi)(f').
\end{multline*}
If the above expression vanishes, then we obtain that $\Phi_{k_1}\#\Xi = 0$ since $\tilde{k}_2$ is integrally strictly positive definite with respect to $\mathcal{M}(H_{k_1})$.

By \citet[Theorem 7]{Simon-GabrielBarpETAL2020}, $\|\cdot\|_{H_{k_1}}$ defined at \eqref{eq:intspd} metrizes weak convergence in $\mathcal{M}_1(X)$, and hence,
$\mathcal{B}(\mathcal{M}_1(X)) = \{\Phi_{k_1}^{-1}(B)\;|\; B \in \mathcal{B}(H_{k_1})\}$. Therefore, $\Xi = 0$, which yields the first part of the theorem. 

Suppose now that $X$ is additionally Polish. Then, $\mathcal{M}_+(X)$ is Polish. Therefore, using that $\Phi_{k_1}$ is measurable and injective, for any Borel set $B \subseteq \mathcal{M}_+(X)$, also $\Phi_{k_1}(B)\subseteq H_{k_1}$ is a Borel set by Purves' theorem \citep{Purves1966}. This implies that $\mathcal{B}(\mathcal{M}_+(X)) = \{\Phi_{k_1}^{-1}(B)\;|\; B \in \mathcal{B}(H_{k_1})\}$. Therefore, $\Xi = 0$, which concludes the proof.
\end{proof}

The proof of the second part of Theorem~\ref{thm:deepkernelM1} actually shows the following, more general statement by setting  $\mathcal{X} = \mathcal{M}_+(X)$, $H = H_{k_1}$ and $T = \Phi_{k_1}$.
\begin{proposition}\label{prop:genT4_WandD}
Let $\mathcal{X}$ be a Polish space, $H$ a separable Hilbert space, $T$ a measurable and injective mapping from $\mathcal{X}$ to $H$, and $\varphi \in \Phi_\infty^+$. 
	Then, the kernel $k$ on $\mathcal{X}$ defined by  
	\begin{equation*}\label{eq:T-Hilbertkernel}
	k(x,x') := \varphi\big(\|T(x)-T(x')\|^2\big), \quad x,x' \in \mathcal{X},
	\end{equation*}
	is integrally strictly positive definite with respect to $\mathcal{M}(\mathcal{X})$. 
\end{proposition}

If there exists a finite Borel measure $\nu$ with full support on a Hausdorff space $(X,\tau)$, then \citet[Corollary 3.6]{SteinwartScovel2012} shows that any bounded and separately continuous kernel $k$ has a separable RKHS $H_k$. If the Hausdorff space in Theorem \ref{thm:deepkernelM1} is compact, then it is sufficient to require $k_1$ to be bounded, continuous and characteristic. In the proof one can then employ \citet[Theorem 5]{Simon-GabrielBarpETAL2020}. For the second part it suffices to note that by \citet[Corollary 3.14]{SteinwartZiegel2021}, $(X,\tau)$ is necessarily Polish.

\citet[Example 2]{ChristmannSteinwart2010} provide a second approach to constructing universal kernels on probability measures over compact subsets of $\mathbb{R}^d$ via comparing Fourier transforms (or characteristic functions). Our results allow for a similar construction on $\mathbb{R}^d$. For a finite measure $\mu \in \mathcal{M}_+(\mathbb{R}^d)$, let $\hat\mu(s) = \int_{\mathbb{R}^d} \exp(i \langle x,s\rangle) \diff \mu(x)$, $s \in \mathbb{R}^d$, denote its Fourier transform. For a complex number $z \in \mathbb{C}$, we denote by $\bar{z}$ its complex conjugate.
\begin{proposition}\label{prop:Fourierkernel}
Let $\lambda$ be a probability measure on $\mathbb{R}^d$ with full support, and $\varphi \in \Phi_\infty^+$. Then, 
\begin{equation*}\label{eq:Fourierkernel}
k(\mu,\nu):= \varphi\big(\|\hat{\mu}-\hat{\nu}\|^2_{L^2}\big), \quad \mu, \nu \in \mathcal{M}_1(\mathbb{R}^d),
\end{equation*}
is an integrally strictly positive definite kernel with respect to $\mathcal{M}(\mathcal{M}_1(\mathbb{R}^d))$, where $\mathcal{M}_1(\mathbb{R}^d)$ is equipped with the topology of weak convergence of probability measures.
\end{proposition}
\begin{proof}
The kernel 
\begin{equation}\label{eq:Fourierkernel2}
\tilde{k}(f,g):= \varphi\big(\|f-g\|^2_{L^2}\big), \quad  f,g \in L^2(\lambda)
\end{equation}
is integrally strictly positive definite with respect to $\mathcal{M}(L^2(\lambda))$ by Corollary \ref{cor:Hilbertkernel}. The map $\Phi$ from $\mathcal{M}_1(\mathbb{R}^d)$ to $L^2(\lambda)$ that maps each probability measure $\mu$ to its Fourier transform $\hat{\mu}$ is injective and continuous by standard properties of Fourier transforms. It is important here that $\lambda$ is a finite measure on $\mathbb{R}^d$. Furthermore, $L^2$-convergence of $(\hat{\mu}_n)_{n \in \mathbb{N}}$ implies weak convergence of $(\mu_n)_{n \in \mathbb{N}}$. Therefore, two measures in $\mathcal{M}(\mathcal{M}_1(\mathbb{R}^d))$ coincide if and only if their image measures under $\Phi$ coincide. 
\end{proof}

A further approach to constructing (strictly) positive definite kernels on absolutely continuous probability distributions on $\mathbb{R}^d$ with finite second moment is given in \citet[Proposition 1]{BachocSuvorikovaETAL2020}. 
They compose a positive definite kernel on $L^2(\R^d)$ with the inverse of the optimal transport maps of the considered probability measures to a reference distribution. The reference distribution is usually chosen as a Wasserstein barycenter. Investigating whether the resulting kernels are integrally strictly positive definite is not straightforward. However, Proposition~\ref{prop:genT4_WandD} can be leveraged to establish the integrally strictly positive definiteness of a class of covariance kernels on Wasserstein space building upon \citep{MerigotDelalandeETAL2020} as developed in the next example. 
\begin{example}[Some integrally strictly positive definite kernels on Wasserstein space]
Let $K,K' \subset \mathbb{R}^d$ be compact and convex with unit Lebesgue measure. For any $\mu \in \mathcal{M}_1(K)$, \citet[Definition 1.1]{MerigotDelalandeETAL2020} denote by $T_{\mu}$ the optimal transport map between the Lebesgue measure on $K'$ and $\mu$, and they call the map $\mu \in \mathcal{M}_{1}(K) \mapsto T_{\mu} \in L^2(K', \mathbb{R}^d)$ the \textit{Monge embedding}. \citet[Theorem 3.1]{MerigotDelalandeETAL2020} show that there exists a constant $C$ (depending on $d$, $K$ and $K'$) such that for all $\mu, \nu \in \mathcal{M}_1(K)$, the following holds: 
\begin{equation*}
\label{eq:merigot}
W_2(\mu,\nu)\leq \tilde{W}_{2}(\mu,\nu) \leq C W_2(\mu,\nu)^{\frac{2}{15}},
\end{equation*}   
where $W_2$ is the $2$-Wasserstein distance on $\mathcal{M}_1(K)$ and the distance $\tilde{W}_{2}$ on $\mathcal{M}_1(K)$ is defined by $\tilde{W}_{2}(\mu,\nu)=\|T_{\mu}-T_{\nu}\|_{L^2}$. This implies in particular that the Monge embedding is injective and continuous on $\mathcal{M}_1(K)$ endowed with the $2$-Wasserstein metric, which is a complete separable metric space \citep[Chapter 2]{PanaretosZemel2020}. Applying Proposition~\ref{prop:genT4_WandD}, we obtain that for any $\varphi \in \Phi_\infty^+$, the kernel 
defined by  
	\begin{equation*}
	k(\mu,\nu):= \varphi\big(\tilde{W}_2^2(\mu,\nu)\big), \quad \mu, \nu \in \mathcal{M}_1(K),
	\end{equation*}
	is integrally strictly positive definite with respect to $\mathcal{M}(\mathcal{M}_1(K))$, where $\mathcal{M}_1(K)$ is equipped with the $2$-Wasserstein topology, which coincides with the topology of weak convergence.
\end{example}

Kernels over measures have been used in the literature to learn kernels of Gaussian processes. When the focus is on radial or translation invariant kernels on $\R^d$, then kernels can be chosen via their spectral measure given by Bochner's theorem \citep[Theorem 6.6]{Wendland2005} which has been advocated for example by \citet{Kom-SamoRoberts2015,BentonMaddoxETAL2019,HamidSchulzeETAL2021}.  

\begin{example}[Translation invariant kernels]\label{ex:transinv}
Bochner's theorem states that a continuous positive definite function on $\mathbb{R}^d \times \mathbb{R}^d$ (with values in $\mathbb{C}$) that depends only on the difference $x-y$ of the arguments $(x,y) \in \mathbb{R}^d \times \mathbb{R}^d$ and is one if $x = y$ is necessarily of the form $\boldsymbol{\varphi}(x-y)$ with $\boldsymbol{\varphi}$ being the Fourier transform of some probability measure $\mu \in \mathcal{M}_1(\mathbb{R}^d)$. 
The arguments in the proof of Proposition \ref{prop:Fourierkernel} show that the set $\mathcal{F}_d$ of all such $\boldsymbol{\varphi}$ is a closed subset of $L^2(\lambda)$ for any probability measure $\lambda$ on $\mathbb{R}^d$ with full support. Therefore, the kernel $\tilde{k}$ given at \eqref{eq:Fourierkernel2} restricted to $\mathcal{F}_d$ is integrally strictly positive definite with respect to $\mathcal{M}(\mathcal{F}_d)$, where $\mathcal{F}_d$ is equipped with $L^2(\lambda)$-norm. Convergence in $L^2(\lambda)$ is equivalent to compact convergence on $\mathcal{F}_d$, and this is in turn equivalent to pointwise convergence. Therefore, the $\sigma$-algebras generated by these three topologies on $\mathcal{F}_d$ coincide, so we could equivalently consider $\mathcal{M}(\mathcal{F}_d)$, where $\mathcal{F}_d$ is equipped with the topology of pointwise convergence. 

\end{example}

\begin{example}[Radial kernels]

It is also possible to construct integrally strictly positive definite kernels on $\Phi_\infty$, that is, on radial positive definite kernels. The functions in $\Phi_\infty$ are all Laplace transforms of measures on $[0,\infty)$. Proceeding analogously to Example \ref{ex:transinv}, one can see $\Phi_\infty$ as a closed subset of $L^2(\lambda)$ for some probability measure $\lambda$ on $[0,\infty)$ with full support, and use \eqref{eq:Fourierkernel} to define integrally strictly positive definite kernels on $\Phi_\infty$. 


\end{example}




\section{Discussion}\label{sec:discussion}

Kernels have been found to be fruitful in a variety of domains, and the presented results on kernels on non-standard spaces are likely to be of interest to extend the applicability of kernel methods.
In particular, we enrich the collection of kernels on Hilbert and Banach spaces with classes enjoying proven integrally strictly positive definiteness, which may be relevant in functional two-sample testing \citep{GrettonBorgwardtETAL2012} and further endeavors in functional data analysis. 
Kernels on measures are of special interest for machine learning with set- and distributional data \citep{MuandetFukumizuETAL2012, Sutherland2016, SzaboSriperumbudurETAL2016, BuathongGinsbourgerETAL2020}, as well as in the evaluation of probabilistic forecasts \citep{SteinwartZiegel2021} with kernel scores associated with characteristic kernels providing a versatile source of strictly proper scoring rules. To give a concrete example of a new avenue of research, scoring rules for point processes are just in their infancy \citep{HeinrichSchneiderETAL2019,BrehmerGneitingETAL2021}, and results established in the present paper pave the way to novel classes of strictly proper (kernel) scoring rules on such objects.       

%
%

\bibliographystyle{plainnat}
\bibliography{biblio}

\appendix

\end{document}